%% file: neurips_2021.tex
\documentclass{article}

\usepackage[nonatbib,preprint]{neurips_2021}
\usepackage{tikz,amsmath,siunitx}

\input{macros}

\title{
 Rademacher Random Projections with Tensor Networks }
\author{
 Beheshteh T. Rakhshan\\
DIRO \& Mila \\
 Université de Montréal\\
  \texttt{rakhshab@mila.quebec}

 \And

     Guillaume Rabusseau \\
   DIRO \& Mila,
  CIFAR AI Chair \\
  Université de Montréal\\
  \texttt{grabus@iro.umontreal.ca}  \\
}

\hypersetup{colorlinks,
    citecolor=blue}
\begin{document}
\maketitle

\begin{abstract}
Random projection~(RP) have recently emerged as popular techniques in the machine learning community for their ability in reducing the dimension of very high-dimensional tensors. Following the work in \cite{rakhshan2020tensorized}, we consider a tensorized random projection relying on Tensor Train (TT) decomposition where each element of the core tensors is drawn from a Rademacher distribution.
Our theoretical results reveal that the Gaussian low-rank tensor represented in compressed form in TT format in \cite{rakhshan2020tensorized} can be replaced by a TT tensor with core elements drawn from a Rademacher distribution with the same embedding size.
Experiments on synthetic data demonstrate that tensorized Rademacher RP can outperform the tensorized Gaussian RP studied in \cite{rakhshan2020tensorized}.
In addition, we show both theoretically and experimentally, that the tensorized RP in the Matrix Product Operator~(MPO) format 
%proposed  in~\cite{feng2020tensor}  and used in \cite{batselier2018computing} for performing SVD on large matrices 
is not a Johnson-Lindenstrauss transform~(JLT) and therefore not a well-suited random projection map.
\end{abstract}

\section{Introduction}
\input{Introduction}
\section{Preliminaries}
\input{preliminaries}
\section{Random Projections based on Tensor Decomposition}

\input{TT_random_projections}

\section{Experiments}
\input{experiments}

\section{Conclusion}
\input{conclusion}
\subsection*{Acknowledgment}
This research was supported by the Canadian Institute
for Advanced Research (CIFAR AI chair program) and the Natural Sciences and Engineering Research Council of Canada (Discovery program,
RGPIN-2019-05949). BR was also supported by an IVADO Fellowship.

\bibliographystyle{plainnat}
\bibliography{biblio.bib}

\end{document}

%% file: macros.tex
\usepackage[utf8]{inputenc} % allow utf-8 input
\usepackage[T1]{fontenc}    % use 8-bit T1 fonts
\usepackage{lmodern}
\usepackage{hyperref}       % hyperlinks
\usepackage{url}            % simple URL typesetting
\usepackage{booktabs}       % professional-quality tables
\usepackage{amsfonts}       % blackboard math symbols
\usepackage{nicefrac}       % compact symbols for 1/2, etc.
\usepackage{microtype}      % microtypography

\usepackage[linesnumbered, ruled, vlined, noend]{algorithm2e}
\usepackage{amsmath}
\usepackage{amsthm}
\usepackage{amssymb}

\usepackage{bm}
\usepackage[title]{appendix}
\usepackage{courier}
\usepackage{enumitem}
\usepackage{graphicx}
\usepackage{algorithmic}
\usepackage[caption=false,font=footnotesize]{subfig}
\usepackage{nicefrac}
\usepackage{mathtools}
\usepackage{color}
\usepackage{xspace} % Correct macro spacing
\usepackage[numbers]{natbib} % For citations
\usepackage{times}
\usepackage{hyperref}
\usepackage{stmaryrd}

\newcommand{\llangle}{\langle\!\langle}
\newcommand{\rrangle}{\rangle\!\rangle}

\usepackage{todonotes}

\newcommand{\fTT}[1]{f_{\mathrm{TT}(#1)}}

\newtheorem*{theorem*}{Theorem}
\newtheorem*{lemma*}{Lemma}
\newtheorem*{prop*}{Proposition}

\newcommand{\ttnormal}{\mathrm{TT}_\mathcal{N}}

\DeclareMathOperator*{\mpo}{\operatorname{MPO}}
\newcommand{\ttrad}{\mathrm{TT}_\text{Rad}}

\RequirePackage{latexsym}
\RequirePackage{amsmath}
\RequirePackage{amssymb}
\RequirePackage{bm}
\RequirePackage[mathscr]{euscript}

\newcommand{\Ex}[1]{\mbox{}\mathbb{E}\left[#1\right]}
\newcommand{\Var}[1]{\mbox{}\textup{Var}\left(#1\right)}

\newcommand{\ts}{\mathsf{T}}
\newcommand{\RR}[2]{\mathbb{R}^{#1 \times #2}} % Real numbers
\newcommand{\vectorize}{\mathrm{vec}}

\newcommand{\St}{\bm{\mathcal{S}}}
\newcommand{\At}{\bm{\mathcal{A}}}
\newcommand{\Bt}{\bm{\mathcal{B}}}
\newcommand{\Xt}{\bm{\mathcal{X}}}
\newcommand{\Mt}{\bm{\mathcal{M}}}
\newcommand{\Gt}{\bm{\mathcal{G}}}
\newcommand{\Tt}{\bm{\mathcal{T}}}

\newcommand{\TT}[1]{\llangle #1 \rrangle}

\newcommand{\ab}{\mathbf{a}}
\newcommand{\bb}{\mathbf{b}}

\newcommand{\ub}{\mathbf{u}}
\newcommand{\vb}{\mathbf{v}}

\newcommand{\xb}{\mathbf{x}}
\newcommand{\yb}{\mathbf{y}}

\newcommand{\Ab}{\mathbf{A}}
\newcommand{\Bb}{\mathbf{B}}

\newcommand{\Ib}{\mathbf{I}}

\newcommand{\Mb}{\mathbf{M}}

\newcommand{\Xb}{\mathbf{X}}

\ifx\BlackBox\undefined
\newcommand{\BlackBox}{\rule{1.5ex}{1.5ex}}  % end of proof
\fi

\ifx\QED\undefined
\def\QED{~\rule[-1pt]{5pt}{5pt}\par\medskip}
\fi

\ifx\proof\undefined
\newenvironment{proof}{\par\noindent{\bf Proof\ }}{\hfill\BlackBox\\[2mm]}
\fi

\ifx\theorem\undefined
\newtheorem{theorem}{Theorem}
\fi

\ifx\example\undefined
\newtheorem{example}{Example}
\fi

\ifx\property\undefined
\newtheorem{property}{Property}
\fi

\ifx\lemma\undefined
\newtheorem{lemma}[theorem]{Lemma}
\fi

\ifx\proposition\undefined
\newtheorem{proposition}[theorem]{Proposition}
\fi

\ifx\remark\undefined
\newtheorem{remark}[theorem]{Remark}
\fi

\ifx\corollary\undefined
\newtheorem{corollary}[theorem]{Corollary}
\fi

\ifx\definition\undefined
\newtheorem{definition}{Definition}
\fi

\ifx\step\undefined
\newtheorem{step}{Step}
\fi

\ifx\conjecture\undefined
\newtheorem{conjecture}[theorem]{Conjecture}
\fi

\ifx\axiom\undefined
\newtheorem{axiom}[theorem]{Axiom}
\fi

\ifx\claim\undefined
\newtheorem{claim}[theorem]{Claim}
\fi

\ifx\assumption\undefined
\newtheorem{assumption}[theorem]{Assumption}
\fi

\newcommand{\EE}{\mathbb{E}} % Expectation
\newcommand{\PP}{\mathbb{P}} % Probability
\renewcommand{\RR}{\mathbb{R}} % Real numbers
\newcommand{\tr}{\mathop{\mathrm{tr}}}

\newcommand{\eg}{\emph{e.g.} }

\newcommand{\nbr}[1]{\left\|#1\right\|}

\newcommand{\abs}[1]{\left|#1\right|}

\newcommand{\inner}[2]{\left\langle #1,#2 \right\rangle}
\newcommand{\norm}[1]{\|#1\|}

\newcommand{\zero}{\mathbf{0}} % Zero

%% file: Introduction.tex
Tensor decompositions are popular techniques used to effectively deal with high-dimensional tensor computations. They recently become popular in the machine learning community for their ability to perform operations on very high-order tensors and successfully have been applied in neural networks~\cite{novikov2014putting,novikov2015tensorizing}, supervised learning~\cite{stoudenmire2016supervised,novikov2016exponential}, unsupervised learning~\cite{stoudenmire2018learning,miller2021tensor,han2018unsupervised}, neuro-imaging~\cite{Zhou2013}, computer vision~\cite{Lu2013} and signal processing~\cite{Cichocki2009,sidiropoulos2017tensor} to name a few. There are different ways of decomposing high-dimensional tensors efficiently. Two most powerful decompositions, CP~\cite{hitchcock1927expression} and Tucker~\cite{tucker1966some} decompositions, can represent very high-dimensional tensors in a compressed form. However, the number of parameters in the Tucker decomposition grows exponentially with the order of a tensor. While in the CP decomposition, the number of parameters scales better, even computing the rank is an NP-hard problem~\cite{hillar2013most,kolda2009tensor}. Tensor Train (TT) decomposition~\cite{oseledets2011tensor} fixed these challenges as the number of parameters grows linearly with the order of a tensor and enjoys efficient and stable numerical algorithms.

In parallel, recent advances in Random Projections (RPs) and Johnson-Lindestrauss (JL) embeddings have succeeded in scaling up classical algorithms to high-dimensional data~\cite{vempala2005random,bingham2001random}. While many efficient random projection techniques have been proposed to deal with high-dimensional vector  data~\cite{achlioptas2002sampling,ailon2009fast,ailon2013almost}, it is not the case for high-order tensors. To address this challenge, it is crucial to find efficient RPs to deal with the curse of dimensionality caused by very high-dimensional data. Recent advances in employing JL transforms for dealing with high-dimensional tensor inputs offer efficient embeddings for reducing computational costs and memory requirements~\cite{rakhshan2020tensorized,jin2019faster,sun2018tensor,malik2020guarantees,li2021statistically,diao2018sketching}. In particular, Feng et al.~\cite{feng2020tensor} propose to use a rank-1 Matrix Product Operator~(MPO) parameterization  of a random projection.  Similarly, Batselier et al.~\cite{batselier2018computing} used the MPO format to propose an algorithm for randomized SVD of very high-dimensional matrices. In contrast, \cite{rakhshan2020tensorized} propose to decompose each row of the random projection matrix using  the TT format to speed up classical Gaussian RP for very high-dimensional input tensors efficiently, without flattening the structure of the input into a vector.  

Our contribution is two-fold. First, we show that tensorizing an RP using the MPO format does not lead to a JL transform by showing that even in the case of matrix inputs, the variance of such a map does not decrease to zero as the size of embedding increases. This is in contrast with the map we proposed in~\cite{rakhshan2020tensorized} which is a valid JL transform. Second, our results demonstrate that the tensorized Gaussian RP in~\cite{rakhshan2020tensorized} can be replaced by a simpler and faster projection using a Rademacher distribution instead of a standard Gaussian distribution. We propose a tensorized RP akin to tensorized Gaussian RP by enforcing each row of a matrix $\Ab\in\RR^{k\times d^N}$ where $k\ll d^N$ to have a low rank tensor structure (TT decomposition) with core elements drawn independently from a Rademacher distribution. Our results show that the Rademacher projection map still benefits from JL transform properties while preserving the same bounds as the tensorized Gaussian RP without any sacrifice in quality of the embedding size. Experiments show that in practice, the performance of the tensorized RP with Rademacher random variables outperforms tensorized Gaussian RP since it reduces the number of operations as it does not require any multiplication. 

%% file: preliminaries.tex
Lower case bold letters denote vectors,~\eg $\ab$, upper case bold letters denote matrices,~\eg $\Ab$, and bold calligraphic letters denote higher order tensors,~\eg $\At$.  The 2-norm of a vector $\vb$ is denoted by $\norm{\vb}_2$ or simply $\norm{\vb}$. The symbol "$\circ$" denotes the outer product~(or tensor product) between vectors. We use $\vectorize(\Mb)\in \RR^{d_1.d_2}$ to denote the column vector
obtained by concatenating the columns of the matrix $\Mb\in\RR^{d_1\times d_2}$. The $d\times d$ identity matrix is denoted by $\Ib_d$. For any integer $i$ we use $[i]$ to denote the set of integers from 1 to $i$. 
\subsection{Tensors}
A tensor $\Tt\in\RR^{d_1\times\dots\times d_N}$ is a multidimensional array and its Frobenius norm is defined by $\nbr\Tt_F^2 = \inner\Tt\Tt$. If $\At\in\RR^{I_1\times\dots\times I_N}$ and $\Bt\in\RR^{J_1\times\dots\times J_N}$, we use $\At\otimes\Bt\in\RR^{I_1J_1\times\dots\times I_NJ_N}$ to denote the Kronecker product of tensors. 
Let $\St\in\RR^{d_1\times\cdots\times d_N}$ be an $N$-way tensor. For $n\in[N]$, let $\Gt^n\in\RR^{R_{n-1}\times d_n\times R_n}$ be 3rd order core tensors with $R_0=R_N=1$ and $R_1=\dots=R_{N-1}=R$. A rank $R$ \textit{tensor train decomposition} of  $\St$ is given by 
$\St_{i_1,\cdots,i_N} = (\Gt^1)_{i_1,:}( \Gt^2)_{:,i_2,:}\cdots(\Gt^{N-1})_{:,i_{N-1},:}(\Gt^N)_{:,i_N}$, for all indices $i_1\in[d_1],\cdots,i_N\in[d_N]$; we will use the notation $\St=\TT{\Gt^1,\Gt^2,\cdots,\Gt^{N-1},\Gt^N}$ to denote the TT decomposition. 

Suppose $\Tt\in\RR^{I_1\times J_1\times \dots\times I_N\times J_N}$. For $n\in[N]$, let  $\At^n\in\RR^{R_{n-1}\times I_n\times J_n\times R_n}$  with $R_0=R_N=1$ and $R_1=\dots=R_{N-1}=R$. A rank $R$ MPO decomposition of $\Tt$ is given by $\Tt_{i_1,j_1,\dots,i_N,j_N} = (\At^1)_{i_1,j_1,:}(\At^2)_{:,i_2,j_2,:}\dots(\At^{N-1})_{:,i_{N-1},j_{N-1},:}(\At^N)_{:,i_N,j_N}$ for all indices $i_1\in[I_1],\cdots,i_N\in[I_N]$ and $j_1\in[J_1],\dots, j_N\in[J_N]$; we will use the notation $\Tt = \mpo((\At^n)_{n=1}^N)$ to denote the MPO format. 
\subsection{Random Projection}
Random projections~(RP) are efficient tools for projecting linearly  high-dimensional data down into a lower dimensional space while preserving the pairwise distances between points. This is the classical result of the Johnson-Lindenstrauss lemma~\citep{johnson1984extensions} which states that any $n$-point set $P\subseteq \RR^d$ can be projected linearly into a $k$-dimensional space with $k =\Omega(\varepsilon^{-2}\log{(n)})$. One of the simplest way to generate such a projection is using a $d\times k$ random Gaussian matrix $\Ab$, i.e., the entries of $\Ab$ are drawn independently from a standard Gaussian distribution with mean zero and variance one. More precisely, for any two points $\ub,\vb\in P\subseteq \RR^d$ the following inequality holds with high probability
\begin{align*}
(1-\varepsilon)\norm{\ub-\vb}^2\leq\norm{f(\ub)-f(\vb)}^2\leq(1+\varepsilon)\norm{\ub-\vb}^2,
\end{align*} 
where $f:\RR^d\to\RR^k\ (k\ll d)$ is a linear map $f(\xb) = \frac{1}{\sqrt{k}}\Ab\xb$ and $\Ab\in\RR^{k\times d}$ is a random matrix. We also call $f$ a Johnson-Lindenstrauss transform (JLT).
To have a JLT, the random projection map $f$ must satisfy the following two properties: (i) Expected isometry, i.e., $\Ex{\norm{f(\xb)}^2}=\norm{\xb}^2$ and (ii) Vanishing variance, that is  $\Var{\|f(\xb)\|^2}$ decreases to zero as the embedding size $k$ increases.
\vspace*{-0.1cm}

%% file: TT_random_projections.tex
\subsection{Matrix Product Operator Random Projection}\label{subsec:MPO-RP}

Classical random projection maps $f:\xb\to\frac{1}{\sqrt{k}}\Ab\xb$ deal with  high-dimensional data using a dense random matrix $\Ab$. Due to  storage and computational constraints, sparse and very sparse RPs have been proposed in~\cite{achlioptas2003database,li2006very}, but even sparse RPs still suffer from the curse of dimensionality and  cannot handle high-dimensional tensor inputs. 
%To alleviate this difficulty, RP maps can be leveraged by enforcing the rows of a random matrix $\Ab$ to be a low-rank tensor structure. 
To alleviate this difficulty, tensor techniques can be used to compress RP maps.
One natural way for this purpose is to compress the dense matrix $\Ab$ with the Matrix Product Operator~(MPO) format~\cite{oseledets2010approximation}. As shown in Figure~\ref{figure:MPO-TN}, relying on the MPO format, we can define a random projection map which embeds any tensor $\Xt\in\RR^{d_1\times\cdots\times d_N}$ into $\RR^{k_1\times\cdots\times k_N}$, where $k = k_1k_2\dots k_N\ll d_1d_2\cdots d_N$ is the embedding dimension. This map is defined element-wise by
\begin{align}
   f(\Xt)_{j_1,\dots,j_N}
    = 
    \frac{1}{\sqrt{R^{N-1}k}}\sum_{i_1,\dots i_N}\!\!\!\!\left(\mpo((\Gt^{n})_{n=1}^N)\right)_{i_1,\dots,i_N,j_1,\dots,j
    _N}\Xt_{i_1,\dots i_N},\label{eq1}
\end{align} 
where $j_n\in[k_n],i_n\in[d_n]$, $\Gt^1\in\RR^{1\times d_1\times R}, \Gt^N\in\RR^{R\times d_N\times 1}$, $\Gt^n\in\RR^{R\times d_n\times k_n \times R}$ for $1<n<N$ and the entries of each core are drawn independently from standard Gaussian distribution. We call the map defined in eqn.~\ref{eq1} an MPO RP. Moreover, by vectorizing $f(\Xt)$ we can consider the RP $f$ as a map from $\RR^{d_1\times\dots\times d_N}\to\RR^k$.  
Particular cases of this general MPO RP formulation have been considered before. Feng et al~\cite{feng2020tensor} consider the case where $R=1$ and the entries of each core are drawn i.i.d from a Rademacher distribution. Batselier et al~\cite{batselier2018computing} consider a MPO RP where $k=k\cdot 1 \cdot 1 \cdot...\cdot 1$ for randomized SVD in the MPO format.\footnote{For simplicity and clarity we assume all ranks are equal where in \cite{batselier2018computing} different ranks for different modes of the MPO are considered.}
\input{MPO-definition}

Even though this map satisfies the expected isometry property, it is not JLT as its variance does not decrease to zero when the size of the random dimension increases. We show these properties in the following proposition by considering the particular case of $N=2$, $k=k.1.1\dots1$. 
\begin{proposition}\label{mpo-variance}
Let $\Xt\in\RR^{d_1\times \dots\times d_N}$. The MPO RP defined in~eqn.~\eqref{eq1} with $k=k\cdot 1 \cdot 1 \cdot...\cdot 1$ satisfies the following properties
 \begin{itemize}
\item $\EE\nbr{f(\Xt)}_2^2 = \nbr{\Xt}_F^2$,

\item $Var{\norm{f(\Xb)}_2^2}= \frac{2}{k}\nbr{\Xb}_F^4+\frac{2}{R}(1+\frac{2}{k})\tr((\Xb^\ts\Xb)^2)$~~~\text{for} $N=2$.
\end{itemize}
\end{proposition}
\begin{proof}
We start by showing the expected isometry property. For a fixed $\kappa\in[k]$, suppose
$\yb_\kappa=\sum_{i_1,\dots i_N}\left(\mpo((\Gt^{n})_{n=1}^N\right)_{i_1,\dots,i_N,\kappa}\Xt_{i_1,\dots i_N}$
and $\St_\kappa = \mpo((\Gt^n)_{n=1}^N)_{:,:,\dots,:,\kappa}$. With these definitions $\yb=[\yb_1,\dots,\yb_k]$ and $f(\Xt)=\frac{1}{\sqrt{R^{N-1}k}}\yb$. As it is shown in~\cite{rakhshan2020tensorized} (e.g., see section 5.1), for $\Tt=\TT{\Mt^1,\dots,\Mt^N}$ with the entries of each core tenors drawn independently from a Gaussian distribution with mean zero and variance one, we have $\EE\langle\Tt\otimes\Tt,\Xt\otimes\Xt\rangle= R^{N-1}\norm{\Xt}_F^2$. Therefore, $\St_\kappa = \TT{\Gt^1,\dots,\Gt^N_{:,:,\kappa}}$ and
$
    \EE[\yb_\kappa^2] 
    =
    \EE[\langle \St_\kappa\otimes\St_\kappa,\Xt\otimes\Xt\rangle]
    =
    R^{N-1}\norm{\Xt}_F^2.
$
From which we can conclude $\EE[\nbr{f(\Xt)}^2]=\frac{1}{R^{N-1}k}\sum_\kappa \EE[\yb_\kappa^2] = \nbr{\Xt}_F^2$.

Now, in order to find a bound for variance of $\norm{\yb}_2^2$ we need first to find a bound for $\EE[\left\|{\yb}\right\|_2^4]$.
For $N=2$, let $\Tt=\mpo(\Gt^1,\Gt^2)$ and $\yb_k=\sum_{i_1,i_2}\Tt_{i_1,i_2,k}\Xb_{i_1,i_2} = \sum_{i_1,i_2}\sum_{r}\Gt^1_{i_1r}\Gt^2_{ri_2k}\Xb_{i_1i_2}$. In terms of tensor network diagrams, we have
$
\yb = 
$
\begin{tikzpicture}[baseline=0.5ex]
\tikzset{tensor/.style = {minimum size = 0.5cm,shape = circle,thick,draw=black,inner sep = 0pt}, edge/.style   = {thick,line width=.4mm},every loop/.style={}}
\def\y{0}
\def\x{0}
\def\inc{2}
\node[tensor,fill=red!50!white] (G1) at (\x,\y+1/2) {\scalebox{0.85}{$\Gt^1$}};
\node[tensor,fill=blue!40!red!60!white] (G2) at (\x+\inc,\y+1/2) {\scalebox{0.85}{$ \Gt^{2}$}};
\node[tensor,fill=blue!40!white] (X) at (\x+1,\y-1/4) {\scalebox{0.85}{$\Xb$}};
\draw[edge] (G1) -- (G2)node [above,midway] {\scalebox{0.6}{\textcolor{gray}{$R$}}};
\draw[edge] (G1) -- (X) node [above,midway] {\scalebox{0.6}{\textcolor{gray}{$d$}}};
\draw[edge] (G2) -- (X) node [above,midway] {\scalebox{0.6}{\textcolor{gray}{$d$}}};
\draw[edge] (\x+2.25,\y+1/2) -- (\x+3,\y+1/2)  node [above,midway] {\scalebox{0.6}{\textcolor{gray}{$k$}}};
\end{tikzpicture}.
By defining a tensor $\Mt\in\RR^{d\times R \times d \times R}$ element-wise via $\Mt_{i_1r_1i_2r_2} = \sum_{j_1,j_2,k}\Xb_{i_1j_1}\Gt^2_{j_1r_1k}\Gt^2_{j_2r_2k}\Xb_{i_2j_2}$, since $\Gt^1\sim\mathcal{N}(\zero,\Ib)$ and by using Isserlis' theorem~\citep{isserlis1918formula} we obtain
\begin{align*}
    &\EE[\nbr{\yb}_2^4] 
    = 
    \EE[\langle(\Gt^1)^{\otimes 4},\Mt^{\otimes 2}\rangle] = \langle\EE[(\Gt^1)^{\otimes 4}],\EE[\Mt^{\otimes 2}]\rangle\\
    &= 
    \sum_{i_1,\dots,i_4}\sum_{r_1,\dots,r_4} \EE[\Gt^1_{i_1r_1}\Gt^1_{i_2r_2}\Gt^1_{i_3r_3}\Gt^1_{i_4r_4}]\EE[\Mt_{i_1r_1i_2r_2}\Mt_{i_3r_3i_4r_4}]\\
    &=\!\!\!
    \sum_{i_1,\dots,i_4\atop r_1,\cdots r_4}\!\!\!\EE\left({ \delta_{i_1i_2}\delta_{i_3i_4}\delta_{r_1r_2}\delta_{r_3r_4}\!\!+\!\delta_{i_1i_3}\delta_{i_2i_4}\delta_{r_1r_3}\delta_{r_2r_4}\!\!+\!\delta_{i_1i_4}\delta_{i_2i_3}\delta_{r_1r_4}\delta_{r_2r_3}}\right)
    \EE[\Mt_{i_1r_1i_2r_2}\Mt_{i_3r_3i_4r_4}].
        \end{align*}
        It then follows that
        \begin{align*}
    \EE[\nbr{\yb}_2^4] 
    &=
    \EE\sum_{i_1,i_3\atop r_1,r_3}\Mt_{i_1r_1i_1r_1}\Mt_{i_3r_3i_3r_3}+\EE\sum_{i_1,i_4\atop r_1,r_4}\Mt_{i_1r_1i_4r_4}\Mt_{i_1r_1i_4r_4}
    +\EE\sum_{i_1,i_2\atop r_1,r_2}\Mt_{i_1r_1i_2r_2}\Mt_{i_2r_2i_1r_1}\\
    &=
    \EE\left[\tr\left(\Xb\Gt^2_{(2)}(\Gt^2_{(2)})^\ts\Xb^\ts\right)\tr\left(\Xb\Gt^2_{(2)}(\Gt^2_{(2)})^\ts\Xb^\ts\right)\right]
     +
     2\EE\sum_{i_1,i_4}\sum_{r_1,r_4}\Mt_{i_1r_1i_4r_4}\Mt_{i_1r_1i_4r_4},
    \end{align*}
where the second term in the last equation  is obtained by using the symmetry property of the tensor $\Mt$, i.e., $\Mt_{i_1r_1i_2r_2} = \Mt_{i_2r_2i_1r_1}$. 
Since $\Gt^2\sim\mathcal{N}(\zero,\Ib)$ and $\Gt^2_{(2)}(\Gt^2_{(2)})^\ts\in\RR^{d\times d}$ is a random symmetric positive definite matrix, by standard properties of the Wishart distribution~(see \textit{e.g.}, Section 3.3.6 of~\citep{gupta2018matrix}) we have
$
    R^2k^2\nbr{\Xb}_F^2+2Rk\tr((\Xb^\ts\Xb)^2)+2\EE\sum_{i_1,i_4}\sum_{r_1,r_4}\Mt_{i_1r_1i_4r_4}\Mt_{i_1r_1i_4r_4}.
$
Again, by using Isserlis' theorem element-wise for the tensor $\Gt^2$, we can simplify the third term in above equation
\begin{align*}
    &\EE\sum_{i_1,i_4}\sum_{r_1,r_4}\Mt_{i_1r_1i_4r_4}\Mt_{i_1r_1i_4r_4}&\\
    &= 
    \EE\sum_{i_1,i_4}\sum_{r_1,r_4}\sum_{j_1,j_2,k_1}\sum_{j_3,j_4,k_2}\left(\Xb_{i_1j_1}\Gt^2_{j_1r_1k_1}\Gt^2_{j_2r_4k_1}\Xb_{i_4j_2}\right)\left(\Xb_{i_1j_3}\Gt^2_{j_3r_1k_2}\Gt^2_{j_4r_4k_2}\Xb_{i_4j_4}\right)\\
    &=
     \EE\sum_{i_1,i_4\atop r_1,r_4}\sum_{j_1,j_2,k_1\atop j_3,j_4,k_2}({\scriptstyle \delta_{j_1j_2}\delta_{j_3j_4}\delta_{r_1r_4}
    +
    \delta_{j_1j_3}\delta_{j_2j_4}\delta_{k_1k_2}
    +
    \delta_{j_1j_4}\delta_{j_3j_2}\delta_{k_1k_2}\delta_{r_1r_4}})
    \Xb_{i_1j_1}  \Xb_{i_4j_2}  \Xb_{i_1j_3}  \Xb_{i_4j_4}\\
    &=
    \EE\!\!\!\!\!\!\sum_{i_1,i_4, r_1\atop j_1,j_3, k_1,k_2}\!\!\!\!\!\!
    \Xb_{i_1j_1}  \Xb_{i_4j_1}  \Xb_{i_1j_3}  \Xb_{i_4j_3}
    +
    \EE\!\!\!\!\!\!\sum_{i_1,i_4,j_1,j_4\atop k_1,k_2}\!\!\!\!\!\!
    \Xb_{i_1j_1}  \Xb_{i_4j_4}  \Xb_{i_1j_1}  \Xb_{i_4j_4}
    +
    \EE\!\!\!\!\!\!\sum_{i_1,i_4,r_1,r_4\atop j_1,j_2  k_1,k_2}\!\!\!\!\!\!
    \Xb_{i_1j_1} \Xb_{i_4j_2}  \Xb_{i_1j_2}  \Xb_{i_4j_1}\\
    &=
    Rk^2\tr((\Xb^\ts\Xb)^2) + kR^2\nbr{\Xb}_F^4+kR\tr((\Xb^\ts\Xb)^2).
    %     \EE\!\!\sum_{i_1,i_4 \atop r_1}\!\!\sum_{j_1,j_3 \atop k_1,k_2}\!\!
    % \Xb_{i_1j_1}  \Xb_{i_4j_1}  \Xb_{i_1j_3}  \Xb_{i_4j_3}
    % \!\!+\!
    % \EE\!\!\sum_{i_1,i_4}\!\!\sum_{j_1,j_4\atop k_1,k_2}\!\!
    % \Xb_{i_1j_1}  \Xb_{i_4j_4}  \Xb_{i_1j_1}  \Xb_{i_4j_4}
    % \!\!+\!\!
    % \EE\!\!\!\sum_{i_1,i_4 \atop r_1,r_4}\!\!\sum_{j_1,j_2 \atop k_1,k_2}\!\!
    % \Xb_{i_1j_1} \Xb_{i_4j_2}  \Xb_{i_1j_2}  \Xb_{i_4j_1}\\
    % &=
    % Rk^2\tr((\Xb\Xb^\ts)) + kR^2\nbr{\Xb}_F^4+kR\tr((\Xb^\ts\Xb)^2).
\end{align*}
Therefore, 
$$\EE[\nbr{\yb}_2^4] = R^2k(k+2)\nbr{\Xb}_F^2+2kR(2+k)\tr((\Xb^\ts\Xb)^2).$$
Finally,
\begin{align*}
    \Var{\nbr{f(\Xb)}_2^2}
    &= 
    \EE[\|k^{-\frac{1}{2}}R^{-\frac{1}{2}}\yb\|_2^4]-\EE[ \|k^{-\frac{1}{2}}R^{-\frac{1}{2}}\yb\|_2^2]^2 =
    \frac{1}{k^2R^2}\EE\nbr{\yb}_2^4-\nbr{\Xb}_F^4\\
    &=
    \frac{1}{k^2R^2}\left(R^2k(k+2)\nbr{\Xb}_F^2+2kR(2+k)\tr((\Xb^\ts\Xb)^2)\right)-\nbr{\Xb}_F^4\\
    &=
    \frac{2}{k}\nbr{\Xb}_F^4+\frac{2}{R}(1+\frac{2}{k})\tr((\Xb^\ts\Xb)^2). \qedhere
\end{align*}
\end{proof}

As we can see for $N=2$, by increasing $k$ the variance does not vanish which validates the fact that the map in~eqn.~(\ref{eq1}) is not a JLT. Using the MPO format to perform a randomized SVD for larges matrices was proposed in~\cite{batselier2018computing} for the first time. As mentioned by the authors, even though numerical experiments demonstrate promising results, the paper suffers from a lack of theoretical guarantees~(e.g., such as probabilistic bounds for the classical randomized SVD~\cite{halko2011finding}). 
The result we just showed in Proposition~\ref{mpo-variance} actually demonstrates that obtaining such guarantees is not possible, since the underlying MPO RP used in \cite{batselier2018computing} is not a JLT. 
As shown in~\cite{rakhshan2020tensorized} this problem can be fixed by enforcing a low rank tensor structure on the rows of the random projection matrix.

\subsection{Tensor Train Random Projection with Rademacher Variables}

We now formally define the map proposed by Rakhshan and Rabusseau~(represented in tensor network diagrams in Figure~\ref{figure:TT-TN}) and show that the probabilistic bounds obtained in~\cite{rakhshan2020tensorized} can be extended to the Rademacher distribution. 

Following the lines in the work done by~\cite{rakhshan2020tensorized} and due to the computational efficiency of TT decomposition, we propose a similar map to $\fTT{R}$ by enforcing a low rank TT  structure on the rows of $\Ab$, where for each row of $\Ab$ the core elements are drawn independently from $\{-1,1\}$ with probability $1/2$, i.e., Rademacher distribution. We generalize and unify the definition of $\fTT{R}$ with Rademacher random projection by first defining the TT distribution and then TT random projection. 
\begin{definition}\label{TTdist}
A tensor $\Tt$ is drawn from a  \emph{TT-Gaussian~(resp. TT-Rademacher) distribution} with rank parameter $R$, denoted by $\Tt\sim \ttnormal(R)$~(resp. $\Tt\sim \ttrad(R)$), if

$$\Tt = \frac{1}{\sqrt{R^{(N-1)}}}\TT{\Gt^1,\Gt^2,\cdots,\Gt^N},$$

where $\Gt^1\in\RR^{1\times d_1\times R} , \Gt^2\in\RR^{R \times d_2 \times R},\cdots, \Gt^{N-1}\in\RR^{R\times d_{N-1}\times R},\Gt^N\in\RR^{R\times d_N\times 1}$ and the entries of each $\Gt^n$ for $n\in[N]$ are drawn independently from the standard normal distribution~(resp. the Rademacher distribution).
\end{definition}

\begin{definition}\label{TTmap}
A \emph{TT Gaussian (resp. TT Rademacher) random projection of rank $R$} is a linear map $\fTT{R}:\RR^{d_1\times\dots  \times d_N}\to\RR^{k}$ defined component-wise by
$$
\left(\fTT{R}(\Xt)\right)_i :=~\frac{1}{\sqrt{kR^{(N-1)}}}\langle \Tt_i,\Xt\rangle,\ i\in[k],
$$
where $\Tt_i\sim\ttnormal(R)$~(resp. $\Tt_i\sim \ttrad(R)$).
\end{definition}

Our main results show that the tensorized Rademacher random projection still benefits from JLT properties as it is an expected isometric map and the variance decays to zero as the random dimension size grows. The following theorems state that using Rademacher random variables instead of standard Gaussian random variables gives us the same results for the bound of the variance while preserving the same lower bound for the size of the random dimension $k$. 
\input{TT-definition}
\begin{theorem}\label{variance}
Let $\Xt\in\RR^{d_1\times d_2 \times \cdots \times d_N}$ and let $\fTT{R}$ be either a tensorized Gaussian RP or a tensorized Rademacher RP of rank $R$~(see Definition~\ref{TTmap}) . The random projection map $\fTT{R}$  satisfies the following properties:

$\bullet~\Ex{\norm{\fTT{R}(\Xt)}_2^2} = \norm{\Xt}_F^2$

$\bullet~\Var{\norm{\fTT{R}(\Xt)}_2^2}\leq\frac{1}{k}(3\left(1+\frac{2}{R}\right)^{N-1}-1)\left\|{\Xt}\right\|_F^4$
\end{theorem}
\begin{proof}

%Since the proof for the TT random projection when $\Tt\sim\ttnormal(R)$ is given in~\cite{rakhshan2020tensorized}, we just bring the proof for $\Tt\sim \ttrad(R)$. However, the expected isometry part is not given here as it employs the exact same proof technique in $\cite{rakhshan2020tensorized}$~(see section 5.1, expected isometry part). 
The proof for the Gaussian TT random projection is given in~\cite{rakhshan2020tensorized}. We now show the result for the tensorized Rademacher RP. The proof of the expected isometry part follows the exact same  technique as in~\cite{rakhshan2020tensorized}~(see section 5.1, expected isometry part), we thus omit it here.
Our proof to bound the variance of $\fTT{R}$ when the core elements are drawn independently from a Rademacher distribution relies on the following lemmas.
\begin{lemma}\label{Isserlis-type}
Let $\Ab\in\RR^{m\times n}$ be a random matrix whose entries are i.i.d Rademacher random variables with mean zero and variance one, and let $\Bb\in\RR^{m\times n}$ be a (random) matrix independent of $\Ab$. Then,
\begin{align*}
\EE\langle\Ab,\Bb\rangle^4\leq3\EE\left\|{\Bb}\right\|_F^4.
\end{align*}
\end{lemma}
\begin{proof}
\renewcommand{\qedsymbol}{$\blacksquare$}
Setting $\ab = \vectorize(\Ab)\in\RR^{mn}$ and $\bb = \vectorize(\Bb)\in\RR^{mn}$, we have
\begin{align*}
    \EE\langle\Ab,\Bb\rangle^4 
    =
    \EE\langle\ab,\bb\rangle^4 =
    \EE\langle\ab^{\otimes4},\bb^{\otimes4}\rangle
    = 
    \sum_{i_1,i_2,i_3,i_4}  \EE[\ab_{i_1},\ab_{i_2},\ab_{i_3},\ab_{i_4}]\EE[\bb_{i_1},\bb_{i_2},\bb_{i_3},\bb_{i_4}],
\end{align*}
we can see that in four cases we have non-zero values for $\EE[\ab_{i_1},\ab_{i_2},\ab_{i_3},\ab_{i_4}]$, i.e., 
\begin{align}\label{fourthmoment}
     \EE[\ab_{i_1},\ab_{i_2},\ab_{i_3},\ab_{i_4}] =    \begin{cases}
      1 & \text{if } i_1=i_2=i_3=i_4 ~~ \text{or}\\
      & i_1=i_2\neq i_3=i_4 ~~ \text{or}\\
      & i_1=i_3\neq i_2=i_4 ~~ \text{or}\\
      &i_1=i_4 \neq i_2=i_3.\\
      0 & \text{otherwise}.
    \end{cases}   
\end{align}
Therefore, 
\begin{align*}
    \EE\langle\Ab,\Bb\rangle^4 
    &=
    \sum_{i_1} \EE[\ab_{i_1}^4]\EE[\bb_{i_1}^4]+
    \sum_{i_1\neq i_3}\EE[\ab_{i_1}^2]\EE[\ab_{i_3}^2]\EE[\bb_{i_1}^2]\EE[\bb_{i_3}^2] +  \sum_{i_1\neq i_4}\EE[\ab_{i_1}^2]\EE[\ab_{i_4}^2]\EE[\bb_{i_1}^2]\EE[\bb_{i_4}^2]\\
    &\ \ \ \ 
    +\sum_{i_1\neq i_2}\EE[\ab_{i_1}^2]\EE[\ab_{i_2}^2]\EE[\bb_{i_1}^2]\EE[\bb_{i_2}^2].
    \end{align*}
Since $\EE[\ab_{i_1}^4]=\EE[\ab_{i_1}^2]=\EE[\ab_{i_2}^2]=\EE[\ab_{i_3}^2]=\EE[\ab_{i_4}^2] = 1$, the equation above can be simplified as 
\begin{align*}
    \EE\langle\Ab,\Bb\rangle^4
    &= 
    \sum_{i_1} \EE[\bb_{i_1}^4]+
    \sum_{i_1\neq i_3}\EE[\bb_{i_1}^2]\EE[\bb_{i_3}^2] +  \sum_{i_1\neq i_4}\EE[\bb_{i_1}^2]\EE[\bb_{i_4}^2] +
    \sum_{i_1\neq i_2}\EE[\bb_{i_1}^2]\EE[\bb_{i_2}^2]\\
    &=
    \EE\sum_{i_1}\bb_{i_1}^4 + \EE\sum_{i_1,i_3}\bb_{i_1}^2\bb_{i_3}^2-\EE\sum_{i_1=i_3}\bb_{i_1}^4+\EE\sum_{i_1,i_4}\bb_{i_1}^2\bb_{i_4}^2\\
    &-
    \EE\sum_{i_1=i_4}\bb_{i_1}^4+\EE\sum_{i_1,i_2}\bb_{i_1}^2\bb_{i_2}^2-\EE\sum_{i_1=i_2}\bb_{i_1}^4
    =
    3\EE\norm{\Bb}_F^4-2\EE\norm{\bb}_4^4\leq 3\EE\norm{\Bb}_F^4.
\end{align*}
\end{proof}\renewcommand{\qedsymbol}{$\blacksquare$}
\begin{lemma}\label{wishart-type}
Let $\Ab\in\RR^{d\times R}$ be a random matrix whose entries are i.i.d Rademacher random variables with mean zero and variance one, and let $\Bb\in\RR^{p\times d}$ be a random matrix independent of $\Ab$, then
$$\EE\norm{\Bb\Ab}_F^4 \leq R(R+2)\EE\norm{\Bb}_F^4.$$

\begin{proof}
Setting $\Mb = \Bb^\ts\Bb$ we have
\begin{align*}
  \EE\nbr{\Bb\Ab}_F^4 
   &= 
   \EE\left[\tr\left(\Bb^\ts\Bb\Ab\Ab^\ts\right)\tr\left(\Bb^\ts\Bb\Ab\Ab^\ts\right)\right] 
   =
   \EE\langle\Mb,\Ab\Ab^\ts\rangle^2\\
   &=
   \!\!\!\!\sum_{i_1,i_2,i_3,i_4}\!\!\!\!\!\EE[(\Ab\Ab^\ts)_{i_1,i_2}(\Ab\Ab^\ts)_{i_3,i_4}]\EE[\Mb_{i_1i_2}\Mb_{i_3i_4}]\\
   &=
   \!\!\!\!\!\!
   \sum_{i_1,i_2,i_3,i_4}\sum_{j,k}\EE[\Ab_{i_1j}\Ab_{i_2j}\Ab_{i_3k}\Ab_{i_4k}]\EE[\Mb_{i_1i_2}\Mb_{i_3i_4}].\\
\end{align*}
Since the components of $\Ab$ are drawn from a Rademacher distribution, the non-zero summands in the previous equation can be grouped in four cases~(which follows from Eq.~\eqref{fourthmoment}):
\begin{align*}
  \EE\nbr{\Bb\Ab}_F^4 
   &= 
   \sum_{i_1\in[d]}\sum_{j,k\in[R]}\EE[\Ab_{i_1j}^2\Ab_{i_1k}^2]\EE[\Mb_{i_1i_1}^2]&(i_1=i_2=i_3=i_4)\\
   &+
   \sum_{i_1\in [d], \atop i_3\in [d] \setminus \{i_1\}}\sum_{j,k\in[R]}\EE[\Ab_{i_1j}^2\Ab_{i_3k}^2]\EE[\Mb_{i_1i_1}\Mb_{i_3i_3}]&(i_1=i_2\neq i_3=i_4)\\
   &+
    \sum_{i_1\in [d], \atop i_2\in [d] \setminus \{i_1\}}\sum_{j,k\in[R]}\EE[\Ab_{i_1j}\Ab_{i_2j}\Ab_{i_2k}\Ab_{i_1k}]\EE[\Mb_{i_1i_2}\Mb_{i_2i_1}]&(i_1=i_4\neq i_2=i_3)\\
   &+
    \sum_{i_1\in [d], \atop i_4\in [d] \setminus \{i_1\}}\sum_{j,k\in[R]}\EE[\Ab_{i_1j}\Ab_{i_4j}\Ab_{i_1k}\Ab_{i_4k}]\EE[\Mb_{i_1i_4}^2]&(i_1=i_3\neq i_2=i_4)
\end{align*}
Now by splitting the summations over $j,k\in[R]$ in two cases $j=k$ and $j\neq k$, and observing that the 3rd and 4th summands  in the previous equation vanish when $j\neq k$, we obtain
\begin{align*}
\MoveEqLeft
   \EE\nbr{\Bb\Ab}_F^4\\
&= 
    \sum_{i_1\in[d]}\sum_{j\in[R]}\EE[\Ab_{i_1j}^4]\EE[\Mb_{i_1i_1}^2]
    + \!\!\!\!
    \sum_{i_1\in[d]}\sum_{j \in [R], \atop k\in[R]\setminus\{j\}}\EE[\Ab_{i_1j}^2\Ab_{i_1k}^2]\EE[\Mb_{i_1i_1}^2]\\
    &+\!\!\!\!
     \sum_{i_1\in [d], \atop i_3\in [d] \setminus \{i_1\}}\sum_{j\in[R]}\EE[\Ab_{i_1j}^2\Ab_{i_3j}^2]\EE[\Mb_{i_1i_1}\Mb_{i_3i_3}]
    % &\hspace{5cm}
    +\!\!\!\!
    \sum_{i_1\in [d], \atop i_3\in [d] \setminus \{i_1\}}\sum_{j \in [R], \atop k\in[R]\setminus\{j\}}\EE[\Ab_{i_1j}^2\Ab_{i_3k}^2]\EE[\Mb_{i_1i_1}\Mb_{i_3i_3}]\\
    &+\!\!\!\!
    \sum_{i_1\in [d], \atop i_2\in [d] \setminus \{i_1\}}\sum_{j\in[R]}\EE[\Ab_{i_1j}^2\Ab_{i_2j}^2]\EE[\Mb_{i_1i_2}\Mb_{i_2i_1}]
    +\!\!\!\!
    \sum_{i_1\in [d], \atop i_4\in [d] \setminus \{i_1\}}\sum_{j\in[R]}\EE[\Ab_{i_1j}^2\Ab_{i_4j}^2]\EE[\Mb_{i_1i_4}^2].
\end{align*}
Since $\EE[\Ab_{i_1j}^4]=1$ and $\EE[\Ab_{i_1j}^2\Ab_{i_1k}^2]=1$ whenever $j\neq k$, it follows that
\begin{align*}
    \EE&\nbr{\Bb\Ab}_F^4 
    =
    R^2\left(\sum_{i_1\in[d]}\!\!\!\EE[\Mb_{i_1i_1}^2] + \!\!\!\!\!\!\sum_{i_1\in [d], \atop i_3\in [d] \setminus \{i_1\}}\!\!\!\!\!\!\!\!\EE[\Mb_{i_1i_1}\Mb_{i_3i_3}]  \right)
    \!\!+\!
    R\!\left(\!\sum_{i_1\in [d], \atop i_2\in [d] \setminus \{i_1\}}\!\!\!\!\!\!\!\!\EE[\Mb_{i_1i_2}\Mb_{i_2i_1}] + \!\!\!\!\!\!\!\sum_{i_1\in [d], \atop i_4\in [d] \setminus \{i_1\}}\!\!\!\!\!\!\!\!\EE[\Mb_{i_1i_4}^2] \right)\\
    &=
    R^2~\EE\sum_{i_1,i_3\in[d]}\Mb_{i_1i_1}\Mb_{i_3i_3} + R~\EE\sum_{i_1,i_2\in[d]}\Mb_{i_1i_2}\Mb_{i_2i_1}
    + R~\EE\sum_{i_1,i_4\in[d]}\Mb_{i_1i_4}^2
    - 2R~\EE\sum_{i_1\in[d]}\Mb_{i_1i_1}^2\\
    &\leq
    R^2~\EE[\tr(\Mb)^2] + 
    R~\EE\sum_{i_1,i_2\in[d]}\Mb_{i_1i_2}\Mb_{i_2i_1} + R~\EE\sum_{i_1,i_4\in[d]}\Mb_{i_1i_4}^2\\
    &=
    R^2\EE[\tr(\Bb^\ts\Bb)^2]  + 
    2R \EE[\tr((\Bb^\ts\Bb)^2)],
\end{align*}
where in the last equation, we used the fact that $\Mb = \Bb^\ts\Bb$ is symmetric. 
Finally, by the submultiplicavity property of the Frobenius norm, we obtain%

\begin{align*}
    \EE\nbr{\Bb\Ab}_F^4 &= R^2 \EE\nbr{\Bb}_F^4 + 2R  \EE\nbr{\Bb^\ts\Bb}_F^2
    \leq 
    R^2\EE\nbr{\Bb}_F^4+2R\EE\nbr{\Bb}_F^4 = R(R+2)\EE\nbr{\Bb}_F^4. \qedhere
\end{align*}
\end{proof}
\end{lemma}

By using these lemmas and the exact same proof technique as in~\cite{rakhshan2020tensorized} one can find the bound for the variance~(\eg see section 5.1, bound on the variance of $\fTT{R}$ part). 
\end{proof}

By employing Theorem~\ref{variance}, Theorem 5 in~\cite{rakhshan2020tensorized} and the hypercontractivity concentration inequality~\cite{schudy2012concentration} we obtain the following theorem which leverages the bound on the variance to give a probabilistic bound on the RP's quality. 
\begin{theorem}\label{JLpropertyCPTT}
Let $P \subset\RR^{d_1\times d_2 \times \cdots \times d_N} $ be a set of $m$ order $N$ tensors. Then, for any $\varepsilon>0$ and any $\delta>0$, the following hold simultaneously for all $\Xt\in P$:

$$
    \PP(\left\|{\fTT{R}(\Xt)}\right\|_2^2 = (1\pm\varepsilon)\left\|{\Xt}\right\|_F^2)\geq 1-\delta~ ~~\text{if}~~~k\gtrsim \varepsilon^{-2}(1+2/R)^N\mathrm{log}^{2N}\left(\frac{m}{\delta}\right).$$
\begin{proof}
The proof follows the one of Theorem 2  in~\cite{rakhshan2020tensorized} \textit{mutatis mutandi}.
\end{proof}
\end{theorem}

%% file: MPO-definition.tex
\begin{figure*}[h!]
\begin{center}
\includegraphics[scale=0.6]{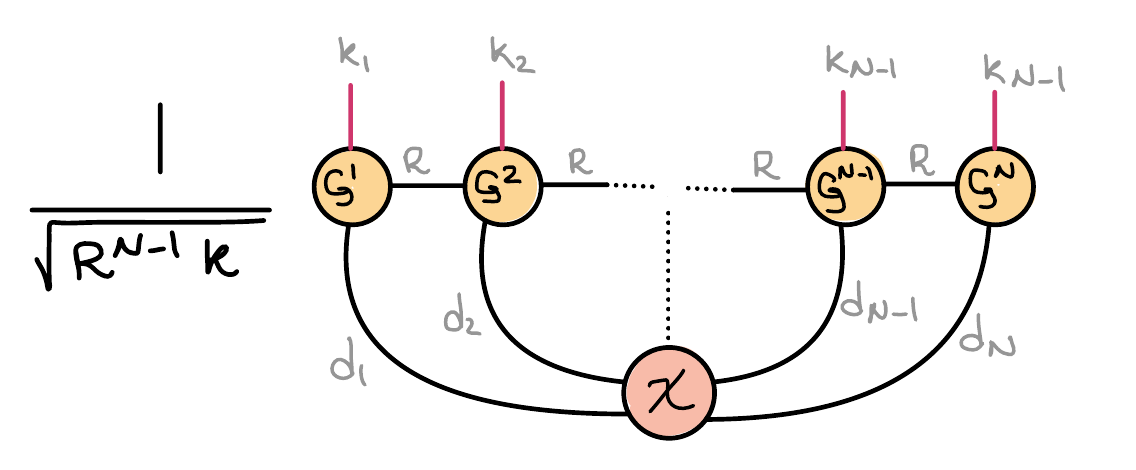}
\caption{Rank $R$ MPO RP in tensor network representation.}
\label{figure:MPO-TN}
\end{center}
\end{figure*}

%% file: TT-definition.tex
\begin{figure*}[h!]
\begin{center}
\includegraphics[scale=0.6]{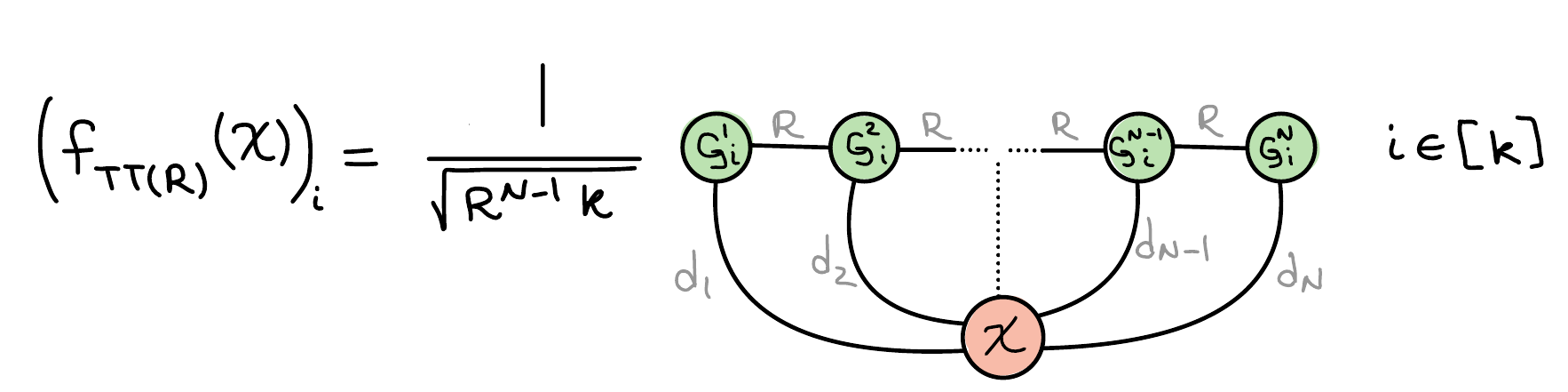}
\caption{Rank $R$ TT RP in tensor network representation.}
\label{figure:TT-TN}
\end{center}
\end{figure*}

%% file: experiments.tex
\input{fig-distortion.tex}
We first  compare the embedding performance of tensorized Rademacher and tensorized Gaussian RPs with classical Gaussian and very sparse~\cite{li2006very} RPs on synthetic data for different size of input tensor and rank parameters. Second, to illustrate that the MPO RPs used in~\cite{batselier2018computing,feng2020tensor} are not well-suited dimension reduction maps, we compare the Gaussian RP $\fTT{R}$ proposed in~\cite{rakhshan2020tensorized} with the MPO RP defined in Section~\ref{subsec:MPO-RP}\footnote{For these experiments we use \textit{TT-Toolbox v2.2}~\cite{TT_Toolbox}.}.
For both parts, the synthetic $N$-th order $d$ dimensional tensor $\Xt$ is generated in the TT format with the rank parameter equals to 10 with the entries of each core tensors drawn independently from the standard Gaussian distribution. 

To compare tensorized Rademacher and Gaussian RPs, following~\cite{rakhshan2020tensorized} we consider three cases for different rank parameters:
small-order $(d=15,N=3)$, medium-order $(d=3,N=12)$ and high-order $(d=3,N=25)$.
The embedding quality of each map is evaluated using the average distortion ratio $D(f,\Xt) = \abs{\frac{\nbr{f(\Xt)}^2}{\nbr{\Xt}^2}-1}$ over 100 trials and is reported as a function of the projection size $k$ in Figure~\ref{figure:distortions}.
Note that due to memory requirements, the high order case cannot be handled with Gaussian or very sparse RPs. As we can see in the small-order case both tensorized maps perform competitively with classical Gaussian RP for all values of the rank parameter. In medium and high order cases, the quality of embedding of the tensorized Rademacher RP outperforms  tensorized Gaussian RP for each value of the rank parameter. Moreover, the tensorized Rademacher RP gives us this speed up as there is no multiplication requirement in the calculations. This is shown in Figure~\ref{figure:timeevaluation}~(right) where we report the time complexity of tensorized Rademacher RP vs tensorized Gaussian RP.
\input{fig-MPO-time}

To validate the theoretical analysis in Proposition~\ref{mpo-variance}, we  consider the higher-order case $(d = 3, N = 25)$ and compare the Gaussian RP $\fTT{R}$ with the MPO RPs proposed in~\cite{batselier2018computing,feng2020tensor} for different values of the rank parameter $R = 1, 5, 10$. These values correspond to roughly the same number of parameters that the two maps require. The quality of embedding via average distortion ratio over 100 trials is reported in Figure~\ref{figure:timeevaluation} where we see that even by increasing the rank parameter of the MPO RPs, the quality of the embedding does not reach acceptable levels which is predicted by our analysis.

%% file: fig-distortion.tex
\begin{figure*}[th!]
\begin{center}
\includegraphics[width=\textwidth]{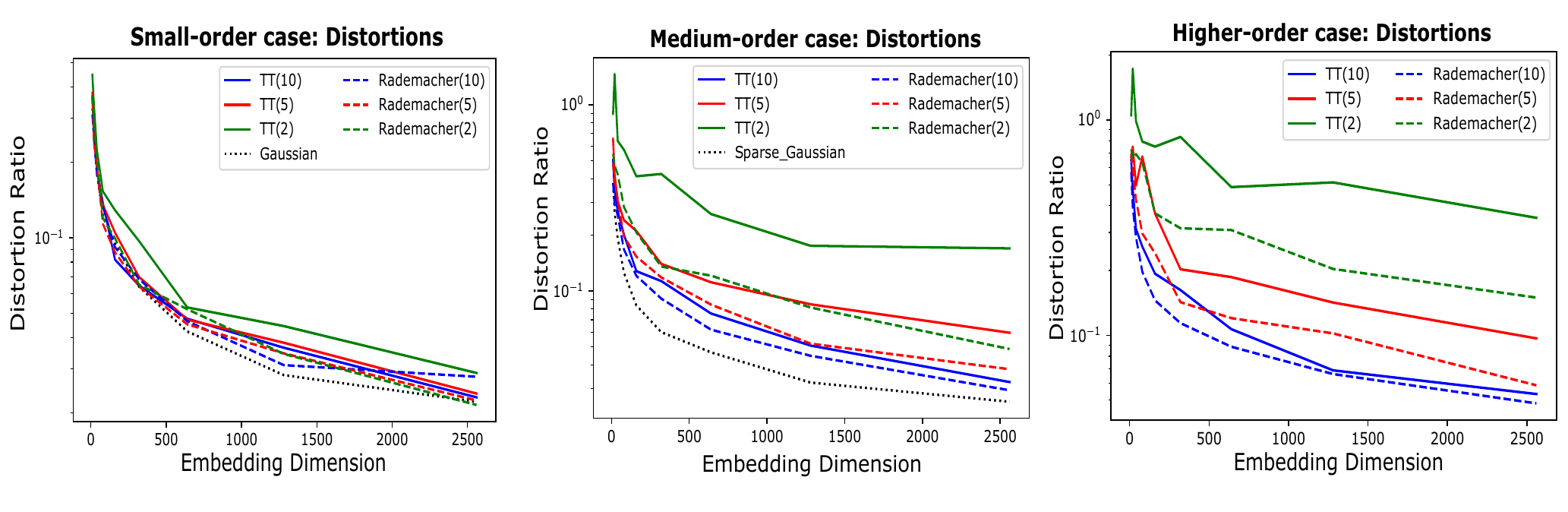}
\caption{Comparison of the distortion ratio of tensorized Rademacher and tensorized Gaussian RPs and classical Gaussian RP for small-order~(left), medium-order~(center) and high-order~(right) input tensors for different value of the rank parameter $R$.}
\label{figure:distortions}
\end{center}
\end{figure*}

%% file: fig-MPO-time.tex
\begin{figure*}[h!]
\begin{center}
\includegraphics[scale=0.9]{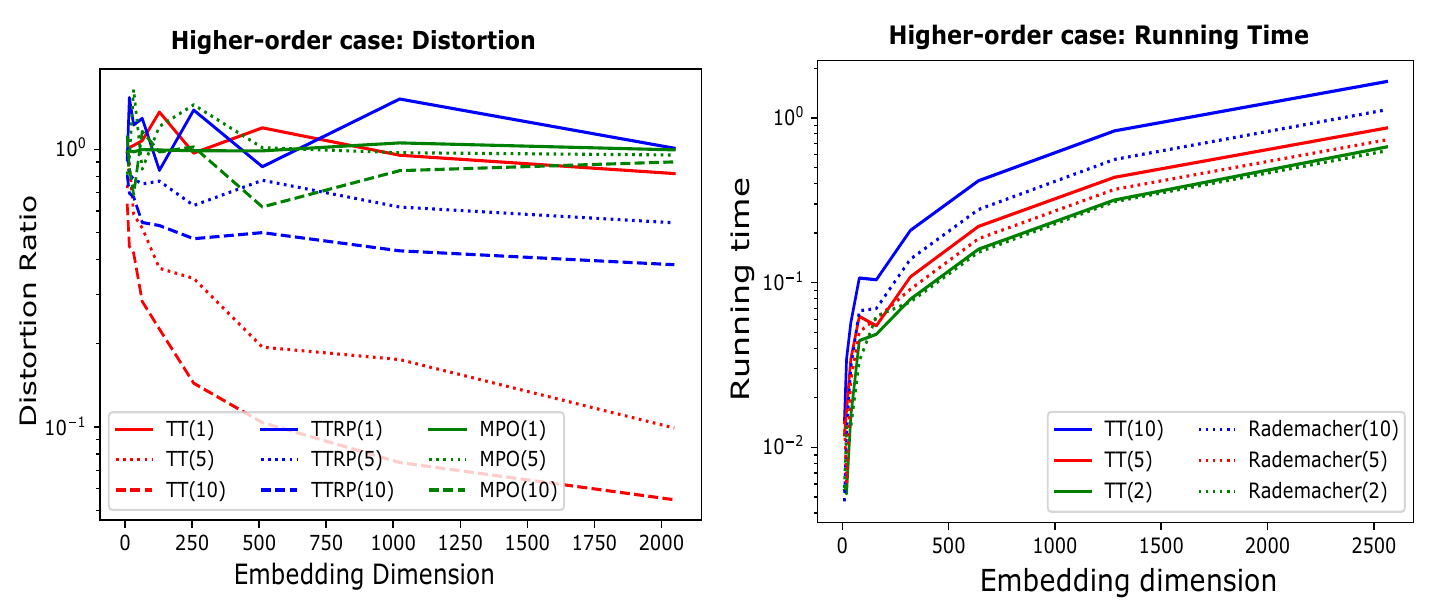}
\caption{Comparison of distortion ratio of tensorized Gaussian RP and MPO RPs proposed in~\cite{feng2020tensor}~(denoted by TTRP) and~\cite{batselier2018computing}~(denoted by MPO) for the  higher-order case with different values for the rank parameter~(left). Comparison of the running times between tensorized Rademacher and tensorized Gaussian RPs~(right).}
\label{figure:timeevaluation}
\end{center}
\end{figure*}

%% file: conclusion.tex
We presented an extension of the tensorized Gaussian embedding proposed in~\cite{rakhshan2020tensorized} for high-order tensors: Tensorized Rademacher random projection map. Our theoretical and empirical analysis show that the Gaussian tensorized RP in~\cite{rakhshan2020tensorized} can be replaced by the tensorized Rademacher RP while still benefiting from the JLT  properties.% without any sacrifice in the size of the random dimension. 
We also showed, both in theory and practice, the RP in an MPO format is not a suitable dimension reduction map. 
Future research directions include leveraging and developing efficient sketching algorithms relying on tensorized RPs to find  theoretical guarantees for randomized SVD and regression problems of very high-dimensional matrices given in the TT format. 